\documentclass{article} 
\usepackage{fullpage,times}
\usepackage{hyperref}
\usepackage{graphicx}
\usepackage{booktabs}
\usepackage{url}
\usepackage{amsmath, amssymb, amsthm}
\usepackage{times}
\usepackage[small]{caption}
\usepackage{algorithm, algorithmic}
\newfloat{algorithm}{tbp}{loa}
\newcommand{\hyp}[1]{\ensuremath{\hplain_{#1}}}
\newcommand{\ohyp}[1]{\ensuremath{\hplain'_{#1}}}
\renewcommand{\H}{\ensuremath{\mathcal{H}}}
\newcommand{\hplain}{\ensuremath{h}}
\newcommand{\err}{\text{\textup{err}}}
\newcommand{\ex}[1]{\ensuremath{X_{#1}}}
\newcommand{\query}[1]{\ensuremath{Q_{#1}}}
\newcommand{\queryp}[1]{\ensuremath{P_{#1}}}
\newcommand{\ind}[1]{\boldsymbol{1} \left\{ #1 \right\}}
\newcommand{\lab}[1]{\ensuremath{Y_{#1}}}
\newcommand{\delay}{\ensuremath{\tau}}
\newcommand{\numex}[1]{\ensuremath{n_{#1}}}

\renewcommand{\P}{\ensuremath{\mathbb{P}}}
\newcommand{\E}{\ensuremath{\mathbb{E}}}
\newcommand{\samp}[1]{\ensuremath{z_{1:{#1}}}}
\newcommand{\hopt}{\ensuremath{\hplain^*}}
\newcommand{\pmin}[2]{\ensuremath{P_{\min,#1}(#2)}}
\newcommand{\order}{\ensuremath{\mathcal{O}}}

\newcommand{\activealg}{\ensuremath{\mathcal{A}}}
\newcommand{\passive}{\ensuremath{\mathcal{P}}}
\newcommand{\Sec}[1]{Section~\ref{sec:#1}}
\newcommand{\Parens}[1]{\left(#1\right)}

\newtheorem{theorem}{Theorem}
\newtheorem{lemma}{Lemma}

\newtheorem*{theoremGen}{Theorem~\ref{thm:generalization}}

\def\eg.{\mbox{e.}\mbox{g.}} 
\def\ie.{\mbox{i.}\mbox{e.}} 

\newcommand{\figsqueeze}{\vskip-6pt}
\title{Para-active Learning}

\author{
Coauthor \\
Affiliation \\
Address \\
\texttt{email} \\
\And
Coauthor \\
Affiliation \\
Address \\
\texttt{email} \\
\AND
Coauthor \\
Affiliation \\
Address \\
\texttt{email} \\
}

\makeatletter
\let\savedtitle=\@title
\let\savedmaketitle=\@maketitle
\makeatother

\begin{document}

\begin{center}
{\LARGE{{\bf{Para-active learning}}}}

\vspace*{.2in}

\begin{tabular}{cccc}
  Alekh Agarwal
&
  L\'{e}on Bottou
&
  Miroslav Dud\'ik
&
  John Langford
  \\
\end{tabular}

\vspace*{.2in}

Microsoft Research\\
New York NY USA\\
\texttt{\{alekha, leonbo, mdudik, jcl\}@microsoft.com}

\vspace*{.2in}

\begin{abstract}

Training examples are not all equally informative. Active learning
strategies leverage this observation in order to massively reduce the
number of examples that need to be labeled.  We leverage the same
observation to build a generic strategy for parallelizing learning
algorithms. This strategy is effective because the search for
informative examples is highly parallelizable and because we show that
its performance does not deteriorate when the sifting process relies on a slightly
outdated model. Parallel active learning is particularly attractive
to train nonlinear models with non-linear representations because
there are few practical parallel
learning algorithms for such models. We report preliminary experiments
using both kernel SVMs and SGD-trained neural networks.
\end{abstract}

\end{center}

\section{Introduction}

The emergence of large datasets in the last decade has seen a growing
interest in the development of parallel machine learning
algorithms. In this growing body of literature, a particularly
successful theme has been the development of distributed optimization
algorithms parallelizing a large class of machine learning algorithms
based on convex optimization. There have been results on parallelizing
batch~\cite{TeoSVL07, RechtReWrNi2011, BoydPaChPeEc2011},
online~\cite{LangfordSmZi09, DekelGiShXi10, AgarwalDu11} and hybrid
variants~\cite{Tera}. It can be argued that these approaches aim to
parallelize the existing optimization procedures and do not exploit
the statistical structure of the problem to the full extent, beyond
the fact that the data is distributed i.i.d. across the compute
nodes. Other
authors~\cite{McDonaldHM10,ZinkevichWeSmLi10,ZhangDuWa2012,ZhangDuWa2013}
have studied different kinds of bagging and model averaging approaches
to obtain communication-efficient algorithms, again only relying on
the i.i.d. distribution of data across a cluster. These approaches are
often specific to a particular learning algorithm (such as the
perceptron or stochastic gradient descent), and model averaging relies
on an underlying convex loss.
A separate line of theoretical research focuses on optimizing
communication complexity in distributed settings when learning
arbitrary hypothesis classes, with a lesser emphasis on the running time
complexity~\cite{BalcanBlFiMa12,DaumePhSsVe12a,DaumePhSsVe12b}.
Our goal here is to cover a broad set of hypothesis classes,
and also achieve short running times to
achieve a given target accuracy, while employing scalable
communication schemes.

The starting point of our work is the observation that in any large
data sample, \emph{not all the training examples are equally
  informative} \cite{mackay-1992}. Perhaps the simplest example is
that of support vector machines where the support vectors form a small
set of informative examples, from which the full-data solution can be
constructed. The basic idea of our approach consists of using
parallelism to sift the training examples and select those worth using
for model updates, an approach closely related to active learning
\cite{cohn-atlas-ladner-1990,fedorov-1972}. Active learning algorithms
seek to learn the function of interest while minimizing the number of
examples that need to be labeled. We propose instead to use active
learning machinery to redistribute the computational effort from the
potentially expensive learning algorithm to the easily parallelized
example selection algorithm.

The resulting approach has several advantages. Active learning
algorithms have been developed both in agnostic settings to work
with arbitrary hypothesis
classes~\cite{cohn-atlas-ladner-1990,BeygelzimerDL09} as well as
in 
settings where they were tailored to specific hypothesis
classes~\cite{bordes-2005}. Building on existing active
learning algorithms allows us to obtain algorithms that work across a
large variety of hypothesis classes and loss functions. This class
notably includes many learning algorithms with non-convex
representations, which are often difficult to parallelize. The
communication complexity of our algorithm is equal to the label
complexity of an active learner with delayed updates.  We provide some
theoretical conditions for the label complexity to be small for a
delayed active learning scheme similar to Beygelzimer et
al.~\cite{BeygelHLZ10}.  On the computational side, the gains of our
approach depend on the relative costs of training a model and
obtaining a prediction (since the latter is typically needed by an
active learning algorithm to decide whether to query a point or not).

In the following section, we present a formal description and a
high-level analysis of running time and communication complexity of
our approach. Two unique challenge arising in distributed settings are
a synchronization overhead and a varying speed with which nodes
process data. Both of them can yield delays in model updating. In
\Sec{analysis}, we theoretically study a specific active learning
strategy and show that its statistical performance is not
substantially affected by delays.  While our method is fully general,
there are two broad classes of problems where we expect our method to
advance state of the art most: learning algorithms with non-linear
training times and learning algorithms based on non-convex
objectives. In \Sec{experiments} we evaluate our approach on kernel
SVMs and neural networks, experimentally demonstrating its
effectiveness in both of these regimes.

\begin{figure}
\begin{center}
\includegraphics[width=.75\linewidth]{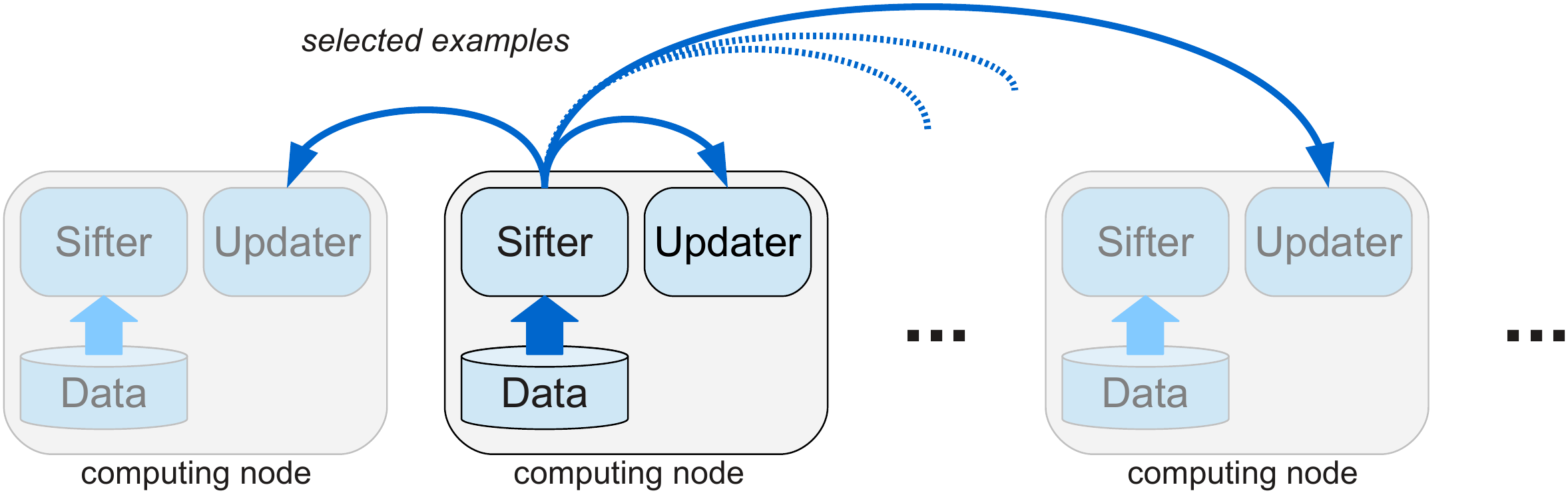}
\caption{\label{fig-paraactive} Parallel active learning.
Each computing node contains an active learner (\emph{sifter})
and a passive learner (\emph{updater}).
The sifter selects interesting training examples and
broadcast them to all nodes. The updater receives
the broadcasts and updates the model. The communication
protocol ensures that examples arrive to each updater in the
same order.}
\end{center}
\figsqueeze
\end{figure}

\begin{algorithm}
\small
  \begin{algorithmic}
    \REQUIRE Initial hypothesis $h_1$, global batch size $B$, active learner
    $\activealg$, passive updater $\passive$.
    \FORALL{rounds $t=1,2,\ldots, T$}
    \FORALL{nodes $i=1,2,\ldots,k$ \emph{in parallel}}
    \STATE Take local data set $X_{i,t}$ with $|X_{i,t}| = B/k$.
    \STATE Obtain $(U_{i,t}, p_{i,t}) = \activealg(X_{i,t}, h_t)$.
    \ENDFOR
    \STATE Let $S_t = \{(U_{i,t}, Y_{i,t}, p_{i,t})~:~1
    \leq i \leq k\}$.
    \STATE Update $h_{t+1} = \passive(S_t, h_t)$.
    \ENDFOR
  \end{algorithmic}
  \caption{Synchronous para-active learning}
  \label{alg:pal-synch}
\end{algorithm}

\section{Parallel active learning}
\label{sec:setup}

In this section we present and analyze our main algorithms
in an abstract setup. Specific instantiations
are then studied
theoretically and empirically in the following sections.

\subsection{Algorithms}
\label{sec:algos}

This paper presents two algorithms, one of which is synchronous and
the other is asynchronous. We will
start with the conceptually simpler synchronous setup in order
to describe our algorithm. We assume there are $k$
nodes in a distributed network, each equipped with its own stream of
data points.

The algorithm operates in two phases, an \emph{active filtering} phase
and a \emph{passive updating} phase. In the first phase, each node
goes over a batch of examples, picking the ones selected by an active
learning algorithm using the current model. Model is not updated in
this phase.  At the end of the phase, the examples selected at all
nodes are pooled together and used to update the model in the
second phase. The second phase can be implemented either at a central
server, or locally at each node if the nodes broadcast the selected
examples over the network. Note that at any given point in time all
nodes have the same model.

\begin{algorithm}
\small
  \begin{algorithmic}
    \REQUIRE Initial hypothesis $h_1$, active learner $\activealg$,
    passive updater $\passive$.
    \STATE Initialize $Q^i_S = \emptyset$ for each node $i$.
    \WHILE{true}
    \FORALL{nodes $i=1,2,\ldots,k$ \emph{in parallel}}

    \WHILE{$Q^i_S$ is not empty}
    \STATE $(x,y,p) = \mbox{fetch}(Q^i_S)$.
    \STATE Update $h^i_\text{new} = \passive((x,y,p),h^i_\text{old})$.
    \ENDWHILE

    \IF{$Q^i_F$ is non-empty}
    \STATE $(x,y) = \mbox{fetch}(Q^i_F)$.
    \STATE Let $p = \activealg(x,h)$ be the probability of $\activealg$ querying $x$
    \STATE With probability $p$:
    \STATE \quad Broadcast $(x,y,p)$ for addition to $Q^j_S$ for all $j$.
    \ENDIF
    \ENDFOR
    \ENDWHILE
  \end{algorithmic}
    \caption{Asynchronous para-active learning}
  \label{alg:pal-asynch}
\end{algorithm}

A critical component of this algorithm is the active learning
strategy.  We use the importance weighted active learning strategy
(IWAL) which has several desirable properties:
consistency, generality~\cite{BeygelzimerDL09}, good rates of
convergence~\cite{BeygelzimerHLZ10} and efficient
implementation~\cite{KarampatziakisL11}.  The IWAL approach operates
by choosing a not-too-small probability of labeling each example and then
flipping a coin to determine whether or not an actual label is asked.

The formal pseudocode is described in Algorithm~\ref{alg:pal-synch}. In
the algorithm, we use $\activealg$ to denote an active learning
algorithm which takes a hypothesis $h$ and an unlabeled example set
$X$ and returns $\activealg(h,X) = (U,p)$ where $U \subseteq X$ and
$p$ is a vector of probabilities with which elements in $X$ were
subsampled to obtain $U$. We also assume access to a passive learning
algorithm $\passive$ which takes as input a collection of labeled
importance weighted examples and the current hypothesis, and returns
an updated hypothesis.

While the synchronous scheme is easier to understand and implement, it
suffers from the drawback that the hypothesis is updated somewhat
infrequently. Furthermore, it suffers from the usual synchronization
bottleneck, meaning one slow node can drive down the performance of
the entire system. Asynchronous algorithms offer a natural solution to
address these drawbacks.

Algorithm~\ref{alg:pal-asynch} is an asynchronous version
of Algorithm~\ref{alg:pal-synch}. It maintains two queues $Q^i_F$ and $Q^i_S$
at each node $i$. $Q^i_F$ stores the fresh examples from the local
stream which haven't been processed yet, while $Q^i_S$ is the queue of
examples selected by the active learner at some node, which need to be
used for updating the model. The algorithm always gives higher
priority to taking examples from $Q^i_S$ which is crucial to its
correct functioning. The communication
protocol ensures that examples arrive to $Q^i_S$ for each $i$ in the
same order. This ensures that models across the nodes agree up to the
delays in $Q^i_S$. See Figure~\ref{fig-paraactive} for a pictorial
illustration.

\subsection{Running time and communication complexity}
\label{sec:complexity}

Consider first an online training algorithm that needs $T(n)$
operations to process $n$ examples to produce a statistically
appropriate model. Apart from this \emph{cumulative training complexity}, we
are also interested in \emph{per-example evaluation complexity} $S(n)$, which
is the time that it takes to evaluate the model on a single example.
For instance, the optimization of a linear model using stochastic
gradient descent requires $T(n){\sim}n$ operations and produces a model
with evaluation complexity $S(n)$ independent of the number of training
examples, \eg. \cite{bottou-2010}. In contrast, training a kernel
support vector machine produces a model with evaluation complexity
$S(n){\sim}n$ and requires at least $T(n){\sim}n^2$ operations to train
(asymptotically, a constant fraction of the examples become support
vectors \cite{steinwart-2004}).

Consider now an example selection algorithm that requires $S(n)$
operations to process each example and decide whether the example
should be passed to the actual online learning algorithm with a
suitable importance weight.  Let $\phi(n)$ be the total number of
selected examples. In various situations, known active learning
algorithms can select as little as $\phi(n){\sim}\log(n)$ and yet
achieve comparable test set accuracy.

Since we intend to sift the training examples in parallel, each
processing node must have access to a fresh copy of the current model.
We achieve this with a communication cost that does not depend on the
nature of the model, by broadcasting all the selected examples. As
shown in Figure~\ref{fig-paraactive}, each processing node can then
run the underlying online learning algorithm on all the selected
examples and update its copy of the model. This requires $\phi(n)$
broadcast operations which can be implemented efficiently using basic
parallel computing primitives.

\begin{figure}
\small\center
\begin{tabular}{lccc}
\toprule
& \bf Sequential Passive & \bf Sequential Active & \bf Parallel Active \\
\midrule
\bf Operations
& $T(n)$ & $n S(\phi(n)) + T(\phi(n))$ & $n S(\phi(n)) + k T(\phi(n))$ \\
\bf Time
& $T(n)$ & $n S(\phi(n)) + T(\phi(n))$ & $n S(\phi(n))/k + T(\phi(n))$ \\
\bf Broadcasts
& $0$ & $0$ & $\phi(n)$ \\
\bottomrule
\end{tabular}
\caption{\label{fig-informal} Number of operations, execution time, and
  communication volume for sequential passive training, sequential active training, and
  parallel active training on $n$ examples and $k$ nodes.
}
\end{figure}

Figure~\ref{fig-informal} gives a sense of how the execution time can
scale with different strategies. Two speedup opportunities arise when
the active learning algorithm selects a number of examples
$\phi(n){\ll}n$ and therefore ensures that $T(\phi(n)){\ll}T(n)$.  The
first speedup opportunity appears when $nS(\phi(n)){\ll}T(n)$ and
benefits both the sequential active and parallel active strategies.
For instance, kernel support vector machines benefit from this speedup
opportunity because $nS(\phi(n)){\sim}n\phi(n){\ll}T(n)$, but neural
networks do not because $nS(\phi(n)){\sim}n{\sim}T(n)$.  The second
opportunity results from the parallelization of the sifting
phase. This speedup is easier to grasp when $nS(n){\sim}T(n)$ as is
the case for both kernel support vector machines and neural networks.
One needs $k{\sim}n/\phi(n)$ computing nodes to ensure that the
sifting phase does not dominate the training time. In other words,
the parallel speedup is limited by both the number of computing nodes
and the active learning sampling rate.
\figsqueeze

\section{Active learning with delays}
\label{sec:analysis}
\figsqueeze

In most standard active learning algorithms, the model is updated as
soon as a new example is selected before moving on to the remaining
examples. Both generalization error and label complexity are
typically analyzed in this setting. However, in the synchronous
Algorithm~\ref{alg:pal-synch}, there can be a delay of as many as $B$
examples ($B/k$ examples on each node) between an example
selection and the model update. Similarly, communication
delays in the asynchronous Algorithm~\ref{alg:pal-asynch} lead to
small variable delays in updating the model. Such delays
could hurt the performance of an active learner. In this section we
demonstrate that this impact is negligible for the particular
importance weighted active learning scheme of Beygelzimer et
al.~\cite{BeygelHLZ10}. While we only analyze this specific
case, it is plausible that the
performance impact is also negligible for other online selective
sampling strategies~\cite{Cesa-Bianchi2009,Orabona2011}.

We now analyze the importance weighted active learning (IWAL) approach
using the querying strategy of Beygelzimer et
al.~\cite{BeygelzimerHLZ10} in a setting with delayed updates.
At a high level, we establish identical generalization
error bounds and show that there is no substantial degradation of the
label complexity analysis as long as the delays are not too
large. We start with the simple setting where the delays are
fixed. Given a time $t$, $\delay(t)$ will be used to denote the
delay until which the labelled examples are
available to the learner. Hence $\delay(t) = 1$ corresponds to
standard active learning.

Algorithm~\ref{alg:delay-iwal} formally describes the
IWAL with delays. Following
Beygelzimer et al.~\cite{BeygelzimerHLZ10}, we let $C_0 =
\order((\log|\H|/\delta)) \geq 2$ be a tuning parameter, while we set
$c_1 = 5+2\sqrt{2}$ and $c_2=5$.
The
algorithm uses the empirical importance weighted error $\err(\hplain,
S_t)$ of hypothesis $\hplain$ on all examples up to (and including)
the example $t - \delay(t)$. Formally, we define
\begin{equation*}
  \err(\hplain, S_t) = \frac{1}{t-\delay(t)}
    \sum_{s=1}^{t-\delay(t)} \frac{\query{s}}{\queryp{s}}
    \ind{\hplain(\ex{s}) \ne \lab{s}},
\end{equation*}
where $\query{s}$ is an indicator of whether we queried the
label $\lab{s}$ on example $s$, $\queryp{s}$ is the probability of
$\query{s}$ being one conditioned on everything up to example $s-1$,
and $\ind{\cdot}$ is the indicator function.

\begin{algorithm}
\small
  \begin{algorithmic}
    \REQUIRE Constants $C_0$, $c_1$, $c_2$.
    \STATE Initialize $S_0 = \emptyset$.

    \FORALL{time steps $t=1,2,\ldots,T$}
    \STATE Let $\hyp{t} = \arg\min\{\err(\hplain, S_t)~:~\hplain \in \H \}$.
    \STATE Let $\ohyp{t} = \arg\min\{\err(\hplain, S_t)~:~\hplain \in \H \wedge \hplain(\ex{t}) \ne \hyp{t}(\ex{t})\}$.
    \STATE Let $G_{t} = \err(\hyp{t}, S_t) - \err(\ohyp{t}, S_t)$, and
    {\small\[
      \queryp{t} = \begin{cases}
      1 & \text{if } G_{t} \leq
                     \sqrt{\frac{C_0\log (t-\delay(t)+1)}{t - \delay(t)}}
                     + \frac{C_0\log( t-\delay(t)+1)}{t-\delay(t)}
      \\
      s & \text{otherwise,}
      \end{cases}
    \]}\relax
    where $s \in (0,1)$  is the positive solution to the equation
    \begin{equation}\textstyle
    \label{eqn:iwal-queryp}
      G_{t} = \left( \frac{c_1}{\sqrt{s}} -c_1 + 1 \right) \cdot
      \sqrt{\frac{C_0\log (t - \delay(t)+1)}{t - \delay(t)}} + \left(
      \frac{c_2}{s} - c_2 + 1 \right) \frac{C_0\log (t -
        \delay(t)+1)}{t - \delay(t)}
    \enspace.
    \end{equation}
    \STATE Query $\lab{t}$ with probability $\queryp{t}$.
    \STATE Let $S_{t+1} = \{
        (\ex{t-s}, \lab{t-s}, \queryp{t - s})
        ~:~ s \ge \delay(t+1)-1 \text{ and $Y_{t-s}$ was queried} \}$.
    \ENDFOR
  \end{algorithmic}
  \caption{Importance weighted active learning with delays}
  \label{alg:delay-iwal}
\end{algorithm}

\subsection{Generalization error bound}

We start with a generalization error bound. It turns out that the theorem of
Beygelzimer et al.~\cite{BeygelzimerHLZ10} applies without major
changes to the delayed setting, even though that is not immediately apparent. The
main steps of the proof are described in Appendix~\ref{app:delayed}.
For convenience, define
$\numex{t} = t - \delay(t)$.
The bound for IWAL with delayed updates takes the following form:
%
\begin{theorem}
\label{thm:generalization}
  For each time $t \geq 1$, with probability at least $1-\delta$ we have
  \begin{equation*}
    0 \leq \err(\hyp{t}) - \err(\hopt) \leq
    \err(\hyp{t},S_t) - \err(\hopt, S_t)
    + \sqrt{\frac{2C_0\log(\numex{t}+1)}{\numex{t}}} +
    \frac{2C_0\log(\numex{t}+1)}{\numex{t}}
  \enspace.
  \end{equation*}
  In particular, the excess risk satisfies
  \begin{equation*}
    \err(\hyp{t}) - \err(\hopt) \leq
    \sqrt{\frac{2C_0\log(\numex{t}+1)}{\numex{t}}} +
    \frac{2C_0\log(\numex{t}+1)}{\numex{t}}
  \enspace.
  \end{equation*}
\end{theorem}

It is easily seen that the theorem matches the previous case of
standard active learning by setting $\delay(t) \equiv 1$ for all $t
\geq 1$. More interestingly, suppose the delays are bounded by
$B$. Then it is easy to see that $\numex{t} = t - \delay(t) \geq t -
B$. Hence we obtain the following corollary in this special case with
probability at least $1 - \delta$
\begin{equation}
    \label{eqn:bounded-det}
    \err(\hyp{t}) - \err(\hopt) \leq \sqrt{\frac{2C_0\log(t - B +
        1)}{t - B}} + \frac{2C_0\log(t - B + 1)}{t - B}
    \enspace.
\end{equation}
As an example, the bounded delay scenario corresponds to a setting
where we go over examples in batches of size $B$, updating the model
after we have collected query candidates over a full batch. In this
case, the delay at an example is at most $B$.

It is also easy to consider the setting of random delays that are
bounded with high probability. Specifically, assume that we
have a random delay process that satisfies:
\begin{equation}
  \label{eqn:delay-rand}
  \P\Parens{\max_{1 \leq s \leq t} \delay(s) > B_t} \leq \delta
  \enspace,
\end{equation}
for some constant $0 < B_t < \infty$. Then it is easy to see that
with probability at least $1 - 2\delta$,
\begin{equation}
  \label{eqn:bounded-rand}
  \err(\hyp{t}) - \err(\hopt) \leq \sqrt{\frac{2C_0\log(t - B_t +
      1)}{t - B_t}} + \frac{2C_0\log(t - B_t + 1)}{t - B_t}
  \enspace.
\end{equation}
Of course, it is conceivable that tighter bounds can be obtained by
considering the precise distribution of delays rather than just a high
probability upper bound.

\subsection{Label complexity}

We next analyze the query complexity. Again, results
of~\cite{BeygelzimerHLZ10} can be adapted to the delayed setting.
Before stating the label complexity bound we
need to introduce the notion of \emph{disagreement
  coefficient}~\cite{Hanneke07} of a hypothesis space $\H$ under a
data distribution $\mathcal{D}$ which characterizes the feasibility of
active learning.  The disagreement coefficient $\theta = \theta(\hopt,
\H, \mathcal{D})$ is defined as
\begin{equation*}
  \theta(\hopt, \H, \mathcal{D}) := \sup\left\{ \frac{\P(X \in
    \mbox{DIS}(\hopt,r))}{r} ~:~ r > 0 \right\},
\end{equation*}
where
\begin{equation*}
  \mbox{DIS}(\hopt, r) := \left\{x \in \mathcal{X}~:~ \exists \hplain
  \in \H~\mbox{such that}~\P(\hopt(X) \ne \hplain(X)) \leq
  r~\mbox{and}~\hopt(x) \ne \hplain(x) \right\}.
\end{equation*}

The following theorem bounds the query complexity of
Algorithm~\ref{alg:delay-iwal}.  It is a consequence of
Lemma~\ref{lemma:label-delay} in Appendix~\ref{app:label:complexity}
(based on a similar result of~\cite{BeygelzimerHLZ10}):
\begin{theorem}
  With probability at least $1 - \delta$, the expected number of label
  queries by Algorithm~\ref{alg:delay-iwal} after $t$ iterations is at
  most
  \begin{equation*}
    1 + 2\theta\, \err(\hopt) \cdot \numex{t} + \order\left(
    \theta\, \sum_{s=1}^t
    \left(\sqrt{\frac{C_0\log(\numex{s}+1)}{\numex{s}}} +
    \frac{C_0\log(\numex{s}+1)}{\numex{s}}\right) \right).
  \end{equation*}
\end{theorem}
Once again, we can obtain direct corollaries in the case of
deterministic and random bounded delays. In the case of delays bounded
determinsitically by $B$, we obtain the natural result that with the
probability at least $1 - \delta$, the query complexity of
Algorithm~\ref{alg:delay-iwal} is at most
\begin{equation*}
  B + 2\theta\, \err(\hopt) \cdot(t-1) + \order\left( \theta \sqrt{t-B}\,
  \sqrt{C_0\log(t)} + \theta \,
  C_0\log(t)\right) .
\end{equation*}

For a random delay process satisfying~\eqref{eqn:delay-rand} the query
complexity is bounded with probability at least $1-2\delta$ by
\begin{equation*}
  B_t + 2\theta\, \err(\hopt) \cdot(t-1) + \order\left( \theta
  \sqrt{t-B_t}\, \sqrt{C_0\log(t)} + \theta \, C_0\log(t) \right).
\end{equation*}

\section{Experiments}
\label{sec:experiments}

In this section we carry out an empirical evaluation of
Algorithm~\ref{alg:pal-synch}.

\textbf{Dataset~~~} In order to experiment with sufficiently large
number of training examples, we report results using the dataset
developed by Loosli et al.~\cite{loosli-canu-bottou-2006}. Each
example in this dataset is a $28\times28$ image generated by
applying elastic deformations to the MNIST training examples.
The first 8.1 million examples of this dataset, henceforth MNIST8M,
are available online.\footnote{\small
  \url{http://leon.bottou.org/publications/djvu/loosli-2006.djvu}}

\textbf{Active sifting~~~}
Our active learning used margin-based
querying~\cite{TongK01,BalcanBZ07}, which is
applicable to classifiers producing real-valued scores $f(x)$
whose sign predicts the target class. Larger absolute values (larger
margins) correspond to larger confidence. A training point $x$ is
queried with probability:
\begin{equation}
  p = \frac{2}{1 + \exp(\eta\, |f(x)|\sqrt{n})}~,
  \label{eqn:actual-queryp}
\end{equation}
where $n$ is the total number of examples
seen so far (including those not selected by the active learner). In
parallel active learning, $n$ is the cumulative number of examples
seen by the cluster until the beginning of the latest sift phase.
The motivation behind this strategy is that in low-noise settings, we
expect the uncertainty in our predictions to shrink at a rate
$O(1/\sqrt{n})$ (or more generally $O(\theta + 1/\sqrt{n})$ if $\theta$
is the Bayes risk). Hence we aim to select examples where we have
uncertainty in our predictions, with the aggressiveness of our
strategy modulated by the constant $\eta$.

\textbf{Parallel simulation~~~} In our experiments we simulate the
performance of Algorithm~\ref{alg:pal-synch} deployed in a parallel
environment. The algorithm is warmstarted with a model trained on a small
subset of examples. We split a global batch into portions
of $B/k$ and simulate the sifting phase of each node in turn.
The queries collected across all nodes in one
round are then used to update the model. We measure the time elapsed
in the sifting phase and use the largest time across all $k$
nodes for each round. We also add the model updating time in each
round and the initial warmstart time. This simulation
ignores communication overhead. However, because of the batched
processing, which allows pipelined broadcasts of all queried examples,
we expect that the communication will be dominated by sifting and updating
times.

\textbf{Support vector machine~~~} The first learning algorithm we
implemented in our framework is kernel SVMs with an RBF kernel. The
kernel was applied to pixel vectors, transformed to lie in $[-1,1]$
following Loosli et al.~\cite{loosli-canu-bottou-2006}. For passive
learning of SVMs, we used the LASVM algorithm of Bordes et
al.~\cite{bordes-2005} with 2 reprocess steps after each new datapoint
to minimize the standard SVM objective in an online fashion.
The algorithm was previously successfully
successfully used on the MNIST8M data, albeit with a different active learning
strategy~\cite{bordes-2005}. The
algorithm was modified to handle importance-weighted queries.

For active learning, we obtain the query probabilities $p$ from the
rule~\eqref{eqn:actual-queryp}, which is then used to obtain
importance weighted examples to pass to LASVM. The importance weight
on an example corresponds to a scaling on the upper bound of the
box constraint of the corresponding dual parameter
and yields $\alpha_i \in [0,C/p]$ instead of the usual $\alpha_i\in [0,C]$
where $C$ is the trade-off parameter for SVMs.
We found that
a very large importance weight can cause instability with the LASVM
update rule, and hence we constrained the change in
$\alpha_i$ for any example $i$ during a process or a reprocess step to
be at most $C$.  This alteration potentially slows the optimization but leaves the
objective unchanged.

We now present our evaluation on the task of distinguishing between
the pair of digits $\{3,1\}$ from the pair $\{5,7\}$. This is expected
to be a hard problem. We set the global batch size to nearly 4\,000 examples,
and the initial warmstart of Algorithm~\ref{alg:pal-synch} is also trained on
approximately 4K examples. The errors reported are MNIST test errors out of a test set of
4065 examples for this task. For all the variants, we use the
SVM trade-off parameter $C = 1$. The kernel bandwidth is set to
$\gamma = 0.012$, where $K(x,y) = \exp(-\gamma\|x-y\|_2^2)$. We ran
three variants of the algorithm: sequential passive, sequential active
and parallel active with a varying number of nodes. For sequential
active learning, we used $\eta = 0.01$ in the
rule~\eqref{eqn:actual-queryp} which led to the best performance,
while we used a more aggressive $\eta = 0.1$ in the parallel setup.

Figure~\ref{fig:error-time} (left) shows how the test error
of these variants decreases as a function of running time.
The running times were measured for the
parallel approach as described earlier.  At a high level, we observe
that the parallel approach shows impressive gains over both sequential
active and passive learning. In fact, we observe in this case that
sequential active learning does not provide substantial speedups over
sequential passive learning, when one aims for a high accuracy, but
the parallel approach enjoys impressive speedups up to 64 nodes. In
order to study the effect of delayed updates from
Section~\ref{sec:analysis}, we also ran the ``parallel
simulation'' for $k=1$, which corresponds to active learning
with updates performed after batches of $B$ examples.
Somewhat surprisingly, this outperformed the
strategy of updating at each example, at least for high accuracies.

\begin{figure}[t]
  \center
  \begin{tabular}{cc}
    \includegraphics[scale=0.35]{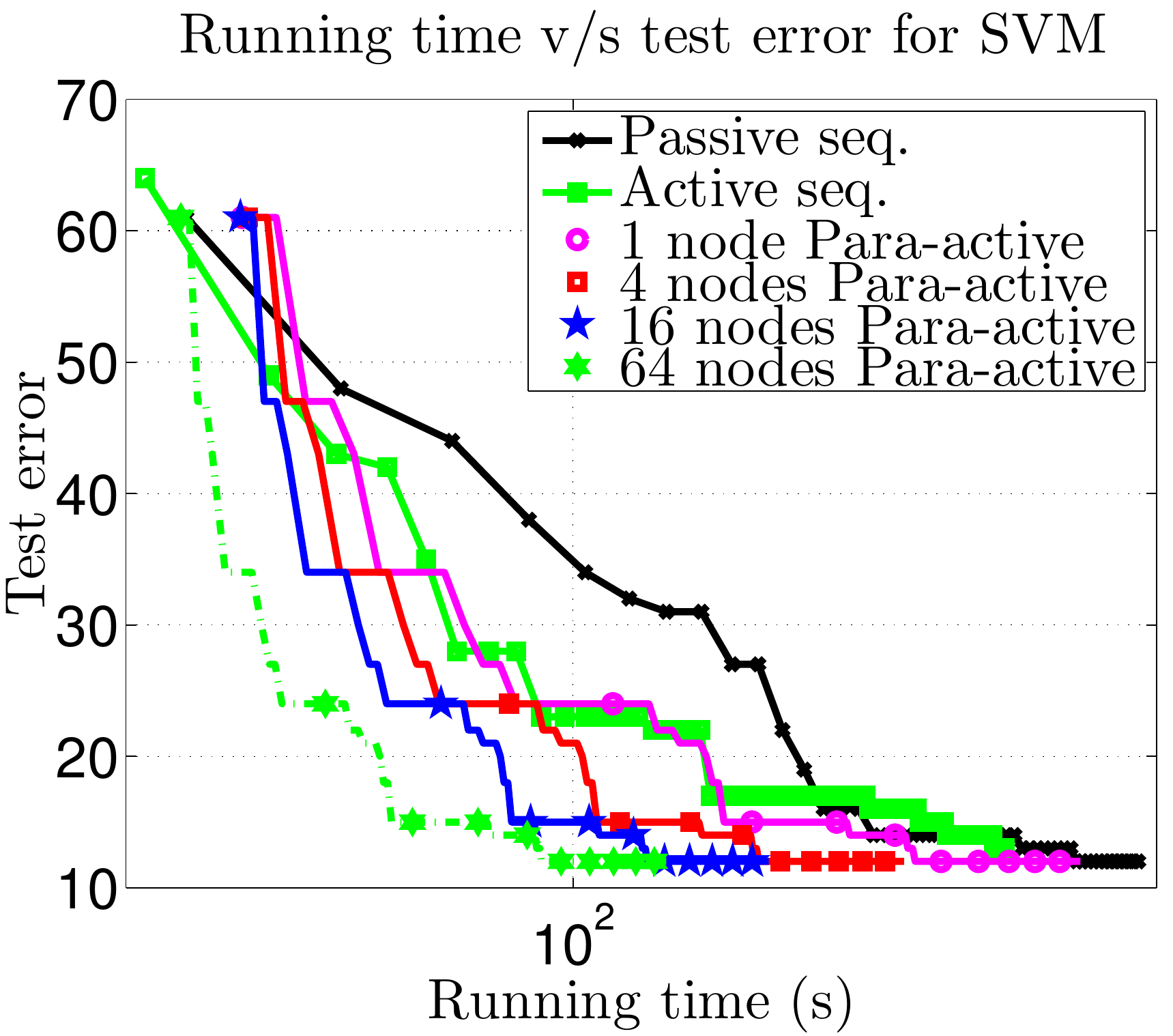} &
    \includegraphics[scale=0.35]{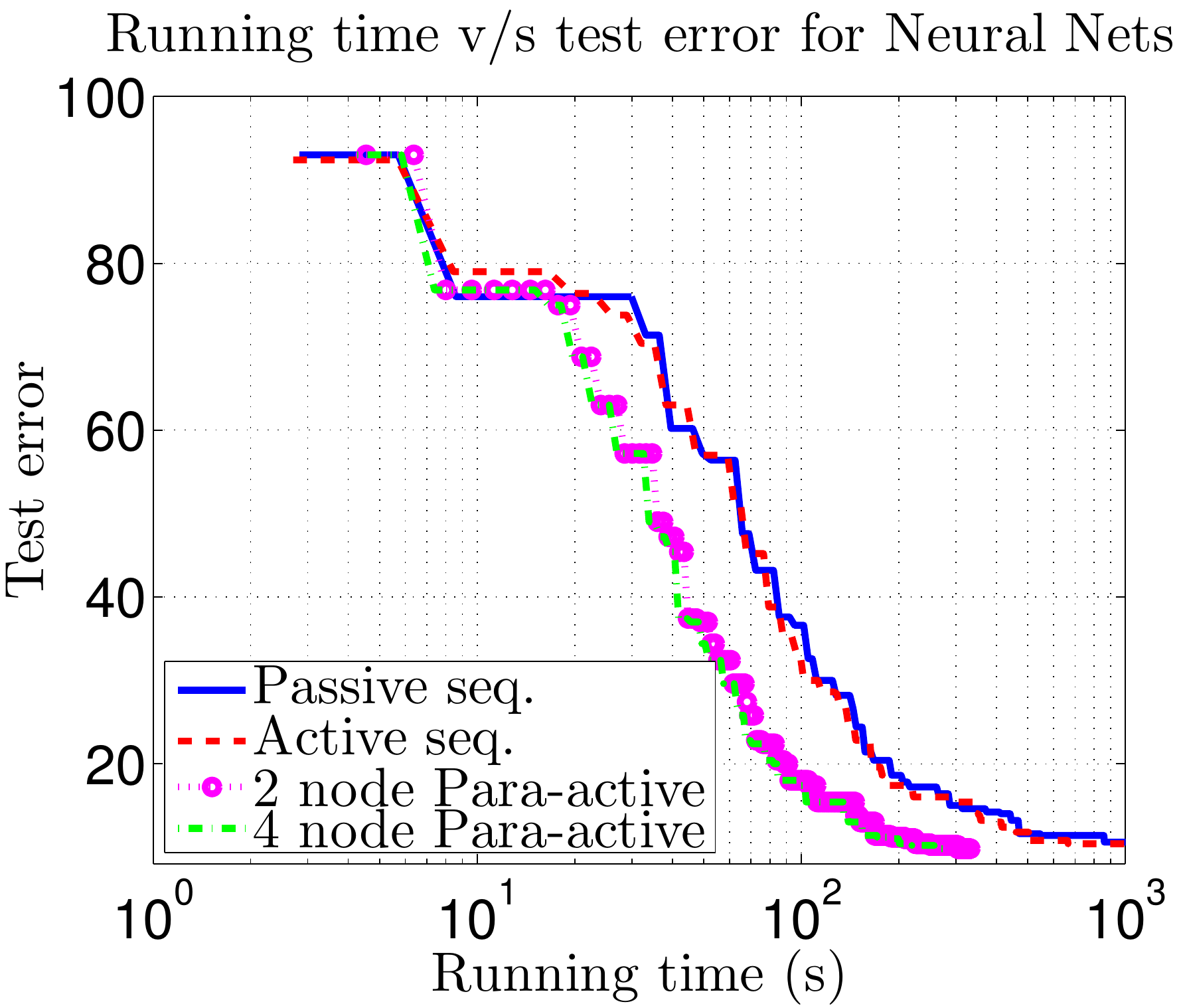}
  \end{tabular}
  \caption{Training time versus test error for passive,
    active, and parallel active learning.}
  \label{fig:error-time}
\end{figure}

To better visualize the gains of parallelization, we plot the speedups
of our parallel implementation over passive learning, and single node
active learning with batch-delayed updates (since that performed
better than updating at each example). The results are shown in
Figure~\ref{fig:svm-speedup}. We show the speedups at several
different levels of test errors (out of 4065 test examples). Observe
that the speedups increase as we get to smaller test errors, which is
expected since the SVM model becomes larger over time (increasing the
cost of active filtering) and the sampling rate decreases.
We obtain substantial speedups until 64
nodes, but they diminish in going from 64 to 128 nodes. This is
consistent with our high-level reasoning of
Figure~\ref{fig-informal}. On this dataset, we found a subsampling
rate of about $2\%$ for our querying strategy which implies that
parallelization over ~50 nodes is ideal.

\begin{figure}[t]
  \center
  \begin{tabular}{cc}
    \includegraphics[scale=0.35]{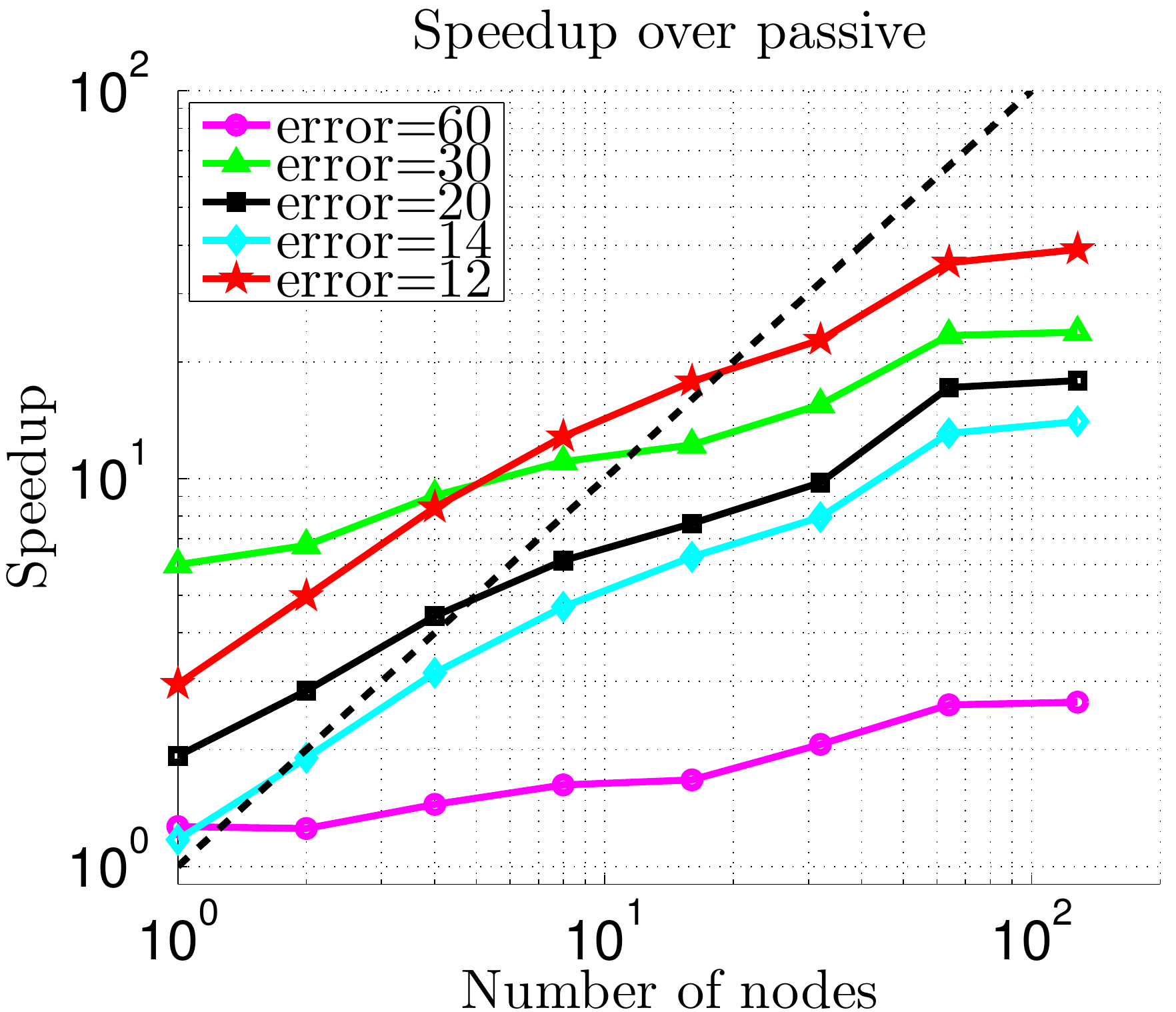} &
    \includegraphics[scale=0.35]{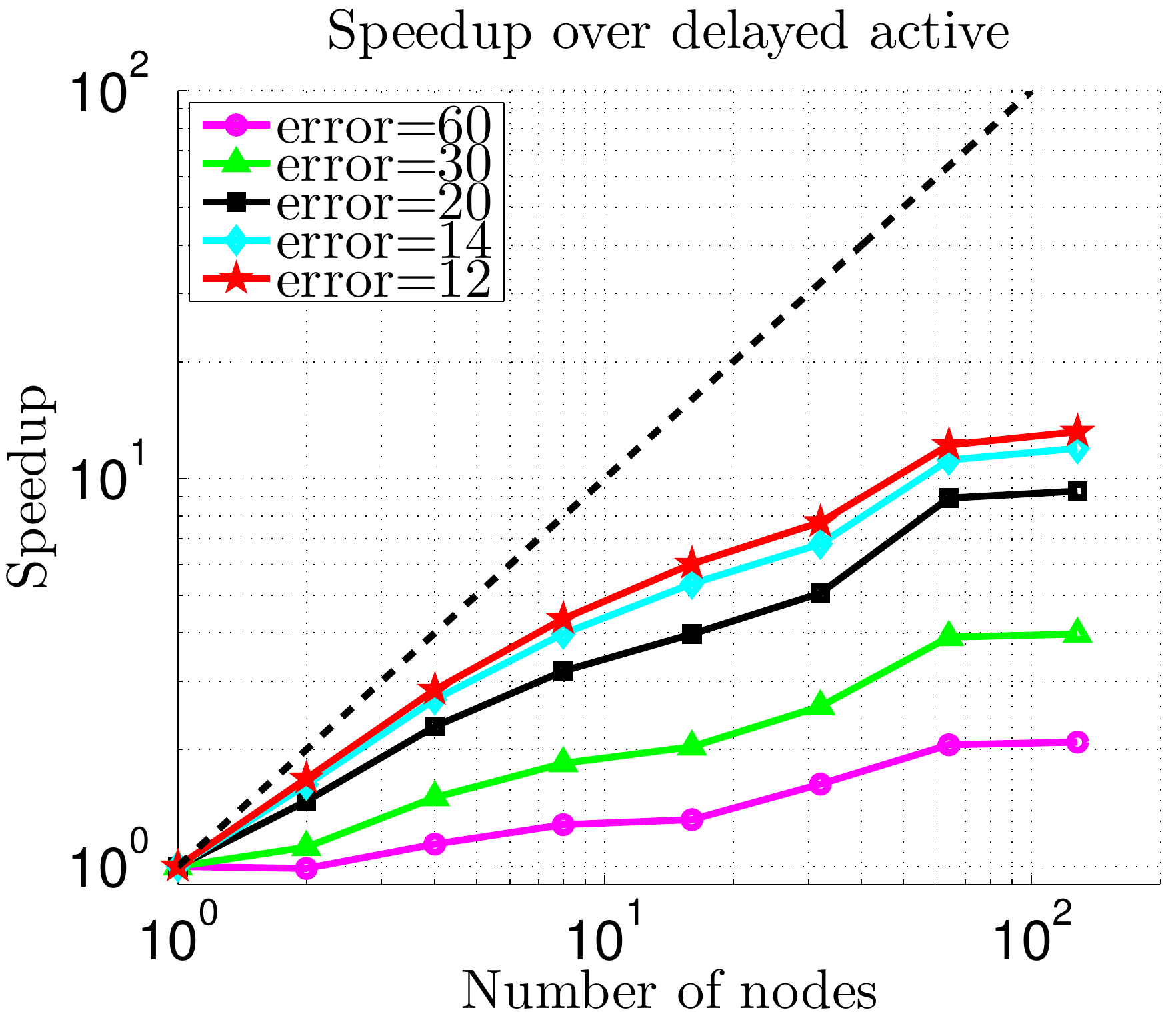}
  \end{tabular}
  \caption{Speedups of parallel active learning over
    passive learning (left) and active learning (right).}
  \label{fig:svm-speedup}
\end{figure}

\textbf{Neural network~~~} With the goal of demonstrating that our
parallel active learning approach can be applied to non-convex problem
classes as well, we considered the example of neural networks with one
hidden layer. We implemented a neural network with 100 hidden nodes,
using sigmoidal activation on the hidden nodes. We used a linear
activation and logistic loss at the output node. The inputs to the
network were raw pixel features, scaled to lie in $[0,1]$. The
classification task used in this case was 3 vs.\ 5. We trained the
neural network using stochastic gradient descent with adaptive
updates~\cite{DuchiHaSi2010, McMahanSt2010}. We used a stepsize of
0.07 in our experiments, with the constant $\eta$ in the
rule~\eqref{eqn:actual-queryp} set to 0.0005. This results in more
samples than the SVM experiments.  Given the modest subsampling rates
(we were still sampling at 40\% when we flattened out at 10 mistakes,
eventually reaching 9 mistakes),
and because the updates are constant-time (and
hence the same cost as filtering), we expect a much less spectacular
performance gain. Indeed, this is reflected in our plots
of Figure~\ref{fig:error-time} (right). While we do see a substantial gain
in going from 1 to 2 nodes,
the gains are modest beyond that as predicted by the 40\% sampling.
A better update rule (which allows
more subsampling) or a better subsampling rule are required for better
performance.



\section{Conclusion}

We have presented a generic strategy to design parallel learning
algorithms by leveraging the ideas and the mathematics of active
learning.  We have shown that this strategy is effective because the
search for informative examples is highly parallelizable and remains
effective when the sifting process relies on slightly outdated models.
This approach is particularly attractive to train nonlinear models
because few effective parallel learning algorithms are available for
such models.  We have presented both theoretical and experimental
results demonstrating that parallel active learning is sound and
effective.  We expect similar gains to hold in practice for all
problems and algorithms for which active learning has been shown
to work.

\bibliographystyle{unsrt}
\bibliography{active_parallel}

\begin{thebibliography}{10}

\bibitem{TeoSVL07}
C.~Hui Teo, A.~J. Smola, S.~V.~N. Vishwanathan, and Q.~V. Le.
\newblock A scalable modular convex solver for regularized risk minimization.
\newblock In {\em KDD}, pages 727--736, 2007.

\bibitem{RechtReWrNi2011}
B.~Recht, C.~Re, S.~Wright, and F.~Niu.
\newblock Hogwild: A lock-free approach to parallelizing stochastic gradient
  descent.
\newblock In {\em NIPS}. 2011.

\bibitem{BoydPaChPeEc2011}
S.~Boyd, N.~Parikh, E.~Chu, B.~Peleato, and J.~Eckstein.
\newblock Distributed optimization and statistical learning via the alternating
  direction method of multipliers.
\newblock {\em Found. Trends Mach. Learn.}, 3(1):1--122, 2011.

\bibitem{LangfordSmZi09}
J.~Langford, A.~Smola, and M.~Zinkevich.
\newblock Slow learners are fast.
\newblock In {\em Advances in Neural Information Processing Systems 22}, pages
  2331--2339, 2009.

\bibitem{DekelGiShXi10}
O.~Dekel, R.~Gilad-Bachrach, O.~Shamir, and L.~Xiao.
\newblock Optimal distributed online prediction using mini-batches.
\newblock {\em ICML}, 2011.

\bibitem{AgarwalDu11}
A.~Agarwal and J.~C. Duchi.
\newblock Distributed delayed stochastic optimization.
\newblock {\em NIPS}, 2011.

\bibitem{Tera}
A.~Agarwal, O.~Chapelle, M.~Dud\'{\i}k, and J.~Langford.
\newblock A reliable effective terascale linear learning system.
\newblock {\em CoRR}, abs/1110.4198, 2011.

\bibitem{McDonaldHM10}
R.~T. McDonald, K.~Hall, and G.~Mann.
\newblock Distributed training strategies for the structured perceptron.
\newblock In {\em HLT-NAACL}, pages 456--464, 2010.

\bibitem{ZinkevichWeSmLi10}
M.~Zinkevich, M.~Weimer, A.~J. Smola, and L.~Li.
\newblock Parallelized stochastic gradient descent.
\newblock In {\em NIPS}. 2010.

\bibitem{ZhangDuWa2012}
Y.~Zhang, J.~Duchi, and M.~Wainwright.
\newblock Communication-efficient algorithms for statistical optimization.
\newblock In {\em NIPS}. 2012.

\bibitem{ZhangDuWa2013}
Y.~Zhang, J.~Duchi, and M.~Wainwright.
\newblock Divide and conquer kernel ridge regression.
\newblock In {\em COLT}, 2013.

\bibitem{BalcanBlFiMa12}
M.-F. Balcan, A.~Blum, S.~Fine, and Y.~Mansour.
\newblock Distributed learning, communication complexity and privacy.
\newblock {\em Journal of Machine Learning Research - Proceedings Track},
  23:26.1--26.22, 2012.

\bibitem{DaumePhSsVe12a}
H.~Daum{\'e} III, J.~M. Phillips, A.~Saha, and S.~Venkatasubramanian.
\newblock Efficient protocols for distributed classification and optimization.
\newblock In {\em ALT}, 2012.

\bibitem{DaumePhSsVe12b}
H.~Daum{\'e} III, J.~M. Phillips, A.~Saha, and S.~Venkatasubramanian.
\newblock Protocols for learning classifiers on distributed data.
\newblock {\em AISTATS}, 2012.

\bibitem{mackay-1992}
D.~J.~C. MacKay.
\newblock Information based objective functions for active data selection.
\newblock {\em Neural Computation}, 4(4):589--603, 1992.

\bibitem{cohn-atlas-ladner-1990}
D.~Cohn, L.~Atlas, and R.~Ladner.
\newblock Training connectionist networks with queries and selective sampling.
\newblock In {\em NIPS}, 1990.

\bibitem{fedorov-1972}
V.~V. Fedorov.
\newblock {\em Theory of Optimal Experiments.}
\newblock Academic Press, New York, 1972.

\bibitem{BeygelzimerDL09}
A.~Beygelzimer, S.~Dasgupta, and J.~Langford.
\newblock Importance weighted active learning.
\newblock In {\em ICML}, 2009.

\bibitem{bordes-2005}
A.~Bordes, S.~Ertekin, J.~Weston, and L.~Bottou.
\newblock Fast kernel classifiers with online and active learning.
\newblock {\em Journal of Machine Learning Research}, 6:1579--1619, September
  2005.

\bibitem{BeygelHLZ10}
A.~Beygelzimer, D.~Hsu, J.~Langford, and T.~Zhang.
\newblock Agnostic active learning without constraints.
\newblock In {\em NIPS}, 2010.

\bibitem{BeygelzimerHLZ10}
A.~Beygelzimer, D.~Hsu, J.~Langford, and T.~Zhang.
\newblock Agnostic active learning without constraints.
\newblock In {\em NIPS}, 2010.

\bibitem{KarampatziakisL11}
N.~Karampatziakis and J.~Langford.
\newblock Online importance weight aware updates.
\newblock In {\em UAI}, pages 392--399, 2011.

\bibitem{bottou-2010}
L.~Bottou.
\newblock Large-scale machine learning with stochastic gradient descent.
\newblock In {\em COMPSTAT'2010}, pages 177--187, 2010.

\bibitem{steinwart-2004}
I.~Steinwart.
\newblock Sparseness of support vector machines---some asymptotically sharp
  bounds.
\newblock In {\em NIPS}. 2004.

\bibitem{Cesa-Bianchi2009}
N.~Cesa-Bianchi, C.~Gentile, and F.~Orabona.
\newblock Robust bounds for classification via selective sampling.
\newblock In {\em ICML}, pages 121--128, 2009.

\bibitem{Orabona2011}
F.~Orabona and N.~Cesa-Bianchi.
\newblock Better algorithms for selective sampling.
\newblock In {\em ICML}, pages 433--440, 2011.

\bibitem{Hanneke07}
S.~Hanneke.
\newblock A bound on the label complexity of agnostic active learning.
\newblock In {\em ICML}, pages 353--360, 2007.

\bibitem{loosli-canu-bottou-2006}
G.~Loosli, S.~Canu, and L.~Bottou.
\newblock Training invariant support vector machines using selective sampling.
\newblock In {\em Large Scale Kernel Machines}. 2007.

\bibitem{TongK01}
S.~Tong and D.~Koller.
\newblock Support vector machine active learning with applications to text
  classification.
\newblock {\em Journal of Machine Learning Research}, 2:45--66, 2001.

\bibitem{BalcanBZ07}
M.-F. Balcan, A.~Z. Broder, and T.~Zhang.
\newblock Margin based active learning.
\newblock In {\em COLT}, pages 35--50, 2007.

\bibitem{DuchiHaSi2010}
J.~Duchi, E.~Hazan, and Y.~Singer.
\newblock Adaptive subgradient methods for online learning and stochastic
  optimization.
\newblock {\em Journal of Machine Learning Research}, 12:2121--2159, 2010.

\bibitem{McMahanSt2010}
H.~B. McMahan and M.~Streeter.
\newblock Adaptive bound optimization for online convex optimization.
\newblock In {\em COLT}, 2010.

\end{thebibliography}

\newpage


\appendix

\section{Generalization bounds for delayed IWAL}
\label{app:delayed}

In this section we provide generalization error
analysis of Algorithm~\ref{alg:delay-iwal}, by showing how to adjust
proofs of Beygelzimer et al.~\cite{BeygelzimerHLZ10}.
To simplify notation, we will
use the shorthand $\epsilon_t = C_0\log(t - \delay(t) + 1)/(t -
\delay(t))$. We start by noting that Lemma 1 of~\cite{BeygelzimerHLZ10}
still applies in our case, assuming we can establish the
desired lower bound on the query probabilities. This forms the
starting point of our reasoning.

In order to state the first lemma, we define the additional notation
$\samp{t-\delay(t)}$ to refer to the set of triples $(X_s, Y_s, Q_s)$ for
$s \leq t - \tau(t)$. Here, $X_s$ is feature vector, $Q_s$ is an
indicator of whether the label was queried, and the
label $Y_s$ is only included on the rounds $s$ where a query was
made. These samples summarize the history of the algorithm up to
the time $t - \tau(t)$ and are used to train $\hyp{t}$. Recall that
$\numex{t} = t - \delay(t)$.

In the following we let $g_t = \err(\ohyp{t},\samp{\numex{t}}) -
\err(\hyp{t},\samp{\numex{t}})$ be the error estimated gap between the
preferred hypothesis at timestep $t$ and the best hypothesis choosing
the other label.  We also let $p(\samp{\numex{t}},x)$ be the
probability of sampling a label when $x$ is observed after history
$\samp{\numex{t}}$ is observed.

We start with a direct analogue of Lemma 1 of Beygelzimer et
al.~\cite{BeygelzimerHLZ10}.
\begin{lemma}[Beygelzimer et al.~\cite{BeygelzimerHLZ10}]
  Pick any $\delta \in (0,1)$ and for all $t \geq 1$ define
  \begin{equation}
    \epsilon_t = \frac{16\log(2(3 + \numex{t}\log_2\numex{t})
      \numex{t}(\numex{t}+1)|\H|/\delta)}{\numex{t}} = \order\left(
    \frac{\log(\numex{t}|\H|/\delta)}{\numex{t}}\right).
    \label{eqn:eps-def}
  \end{equation}
  Suppose that the bound $p(\samp{\numex{t}},x) \geq
  1/\numex{t+1}^{\numex{t+1}}$ is satisfied for all
  $(\samp{\numex{t}},x) \in (\mathcal{X}\times \mathcal{Y}\times
  \{0,1\})^{\numex{t}} \times \mathcal{X}$ and all $t \geq 1$. Then
  with probability at least $1-\delta$ we have for all $t \geq 1$ and
  all $\hplain \in \H$,
  \begin{equation}
    |(\err(\hplain, S_t) - \err(\hopt, S_t))
    - (\err(\hplain) - \err(\hopt))| \leq
    \sqrt{\frac{\epsilon_t}{\pmin{i}{\hplain}}} +
    \frac{\epsilon_t}{\pmin{t}{\hplain}},
    \label{eqn:abs-dev}
  \end{equation}
  where $\pmin{t}{\hplain} = \min\{\queryp{s}~:~ 1 \leq s \leq
  \numex{t} \wedge \hplain(\ex{s}) \ne \hopt(\ex{s})\}$.
  \label{lemma:beygel}
\end{lemma}
In order to apply the lemma, we need the following analogue of Lemma 2
of~\cite{BeygelzimerHLZ10}.
\begin{lemma}
  The rejection threshold of Algorithm~\ref{alg:delay-iwal} satisfies
  $p(\samp{\numex{t}},x) \geq 1/{\numex{t+1}}^{\numex{t+1}}$ for all
  $t \geq 1$ and all $(\samp{\numex{t}},x) \in (\mathcal{X}\times
  \mathcal{Y}\times\{0,1\})^{\numex{t}} \times \mathcal{X}$.
  \label{lemma:plowerbound}
\end{lemma}
\begin{proof}
  The proof is identical to that of~\cite{BeygelzimerHLZ10},
  essentially up to replacing $n$ with appropriate values of
  $\numex{t}$. We proceed by induction like their lemma. The claim for
  $t = 1$ is trivial since the $p(\emptyset,x) = 1$. Now we assume the
  inductive hypothesis that $p(\samp{\numex{s}},x) \geq
  1/\numex{s+1}^{\numex{s+1}}$ for all $s \leq t-1$.

  Note that we can assume that $\numex{t+1} \geq \numex{t} + 1$. If
  not, then $\numex{t+1} = \numex{t}$ and the claim at time $t$
  follows from the inductive hypothesis. If not, then the probability
  $p(\samp{\numex{t}},x)$ for any $x$ is based on the error difference
  $g_t$. Following their argument and the definition of
  Algorithm~\ref{alg:delay-iwal}, one needs to only worry about the
  case where $g_{t} > \sqrt{\epsilon_{t}} +
  \epsilon_{t}$. Furthermore, by the inductive hypothesis we have the
  upper bound $g_t \leq 2(\numex{t})^{\numex{t}}$. Mimicking their
  argument from hereon results in the following lower bound on the
  query probability $p_{i,j}$
  \[
    \sqrt{p_{t}}
    ~~>~~ \sqrt{\frac{c_2\epsilon_{t}}{c_1g_{t}}}
    ~~=~~ \sqrt{\frac{c_2\log(\numex{t}+1)}{c_1\, \numex{t}g_{t}}}
    ~~\geq~~ \sqrt{\frac{c_2\log(\numex{t}+1)}{2c_1\, \numex{t} \numex{t}^{\numex{t}}}}
    ~~>~~ \sqrt{\frac{1}{e\, \numex{t}^{\numex{t}+1}}}
  \enspace.
  \]

  Recall our earlier condition that $\numex{t+1} \geq
  \numex{t}+1$. Hence we have
  \begin{equation*}
    \numex{t}^{\numex{t}+1}
    ~\leq~~ \numex{t}^{\numex{t+1}}
    ~=~~ \numex{t+1}^{\numex{t+1}} \, \left(
    \frac{\numex{t}}{\numex{t+1}}\right)^{\numex{t+1}}
    ~\leq~~ \numex{t+1}^{\numex{t+1}} \, \left(
    \frac{\numex{t+1}-1}{\numex{t+1}}\right)^{\numex{t+1}}
    ~\leq~~ \frac{\numex{t+1}^{\numex{t+1}}}{e}.
  \end{equation*}
  Combining the above two results yields the statement of the lemma.
\end{proof}
Combining the two lemmas yields Theorem~\ref{thm:generalization},
a natural generalization of the result of~\cite{BeygelzimerHLZ10}.
%
\begin{theoremGen}
  For each time $t \geq 1$, with probability at least $1-\delta$ we have
  \begin{equation*}
    0 \leq \err(\hyp{t}) - \err(\hopt) \leq
    \err(\hyp{t},S_t) - \err(\hopt, S_t)
    + \sqrt{\frac{2C_0\log(\numex{t}+1)}{\numex{t}}} +
    \frac{2C_0\log(\numex{t}+1)}{\numex{t}}
  \enspace.
  \end{equation*}
  In particular, the excess risk satisfies
  \begin{equation*}
    \err(\hyp{t}) - \err(\hopt) \leq
    \sqrt{\frac{2C_0\log(\numex{t}+1)}{\numex{t}}} +
    \frac{2C_0\log(\numex{t}+1)}{\numex{t}}
  \enspace.
  \end{equation*}
\end{theoremGen}
\begin{proof}[Proof of Theorem~\ref{thm:generalization}]
  In order to establish the statement of the theorem from
  Lemma~\ref{lemma:beygel}, we just need to control the minimum
  probability over the points misclassified relative to $\hopt$,
  $\pmin{t}{\hyp{t}}$. In order to do so, we observe that the proof of
  Theorem 2 in~\cite{BeygelzimerHLZ10} only relies on the fact that
  query probabilities are set based on an equation of the
  form~\eqref{eqn:iwal-queryp}. Specifically, their proof establishes
  that assuming we have $G_t = (c_1/\sqrt{s} - c_1 +
  1)\sqrt{\epsilon_t} + (c_2/s - c_2 + 1)\epsilon_t$ for the same
  sequence $\epsilon_t$ coming from Lemma~\ref{lemma:beygel}, then the
  statement of the theorem holds. Since this is exactly our setting,
  the proof applies unchanged yielding the desired theorem statement.
\end{proof}

\section{Label complexity lemma}
\label{app:label:complexity}

In this section we derive a natural generalization of the key
lemma~\cite{BeygelzimerHLZ10} for bounding the query complexity.
\begin{lemma}
  \label{lemma:label-delay}
  Assume the bounds from Equation~\ref{eqn:abs-dev} hold for all
  $\hplain \in \H$ and $t \geq 1$. For any $t \geq 1$,
  \begin{equation*}
    \E[\query{t}] \leq \theta\cdot\, 2\,\err(\hopt) + \order\left(
    \theta\cdot \sqrt{\frac{C_0\log(\numex{t}+1)}{\numex{t}}} +
    \theta\cdot \frac{C_0\log(\numex{t}+1)}{\numex{t}} \right).
  \end{equation*}
\end{lemma}
\begin{proof}
  The proof of this lemma carries over unchanged from Beygelzimer et
  al.~\cite{BeygelzimerHLZ10}. A careful inspection of their
  proof shows that they only require $\epsilon_t$ defined in
  Equation~\ref{eqn:eps-def} with query probabilities chosen as in
  Equation~\ref{eqn:iwal-queryp}. Furthermore, we need the statements
  of Lemma~\ref{lemma:beygel} and
  Theorem~\ref{thm:generalization} to hold with the same setting
  of $\epsilon_t$. Apart from this, we only need the sequence
  $\epsilon_t$ to be monotone non-increasing, and $\hyp{t}, \ohyp{t}$
  to be defined based on samples $\samp{\numex{t}}$. Since all these
  are satisfied in our case with $\numex{t}$ appropriately redefined
  to $t - \delay{(t)}$, we obtain the statement of the lemma by
  appealing to the proof of~\cite{BeygelzimerHLZ10}.
\end{proof}

\end{document}